\documentclass{article}
\usepackage[tight]{zshmacros}

\newcommand{\bx}{\bar{x}}

\newcommand{\tA}{\tilde{A}}
\newcommand{\pa}{\tilde{a}}
\newcommand{\pb}{\tilde{b}}
\newcommand{\pc}{\tilde{c}}

\newcommand{\bP}{\bar{P}}

\newcommand{\bR}{\bar{R}}
\newcommand{\tR}{\tilde{R}}
\newcommand{\hR}{\hat{R}}

\newcommand{\bK}{\bar{K}}
\newcommand{\hK}{\hat{K}}
\newcommand{\tK}{\tilde{K}}

\newcommand{\vln}{\widetilde{\ln}}

\begin{document}
\title{Stochastic gradient descent algorithms for strongly convex functions \\
  at $O(1/T)$ convergence rates}
\author{Shenghuo Zhu\\
  \texttt{zsh@nec-labs.com}}
\date{}
\maketitle

\begin{abstract}
  With a weighting scheme proportional to $t$, a traditional
  stochastic gradient descent (SGD) algorithm achieves a high
  probability convergence rate of $O(\kappa/T)$ for strongly convex
  functions, instead of $O(\kappa \ln(T)/T)$. We also prove that an
  accelerated SGD algorithm also achieves a rate of $O(\kappa/T)$.
\end{abstract}

\section{Introduction}
Consider a stochastic optimization problem
\begin{equation*}
  \min_{x \in \calX} \{ f(x) := \expect[\xi] F(x,\xi) \}
\end{equation*}
where $\calX \subset \bbR^d$ is a nonempty bounded closed convex set,
$\xi$ is a random variable, $F$ is a smooth convex function, $f$ is a
smooth strongly-convex function.  The requirement of smoothness
simplifies the analysis. If the objective function is nonsmooth but
satisfies Lipschitz continuity, stochastic gradient descent
algorithms can replace gradients with subgradients, but the analysis
has to introduce an additional term in the same order as the variance
term. Some nonsmooth cases have been studied in
\cite{lan08:_effic_method_stoch_compos_optim} and
\cite{ghadimi12:_optim_stoch_approx_algor_stron}.

Assume that the domain is bounded, i.e. $\sup_{x,y\in \calX} \|x-y\|^2
\leq D^2$.  Let $G(x,\xi)$ be a stochastic gradient of function $f$ at
$x$ with a random variable $\xi$.  Then $g(x):= \expect[\xi] G(x,\xi)$
is a gradient of $f(x)$.  Assume that $\|g(x)-g(y)\|_* \leq L\|x-y\|$,
where $L$ is known as the Lipschitz constant. We only consider
strongly convex function in this note, thus assume that there is $\mu
> 0$, such that $f(y)\geq
f(x)+\left<g(x),y-x\right>+\frac{\mu}{2}\|y-x\|^2$.  We assume that
stochastic gradients are bounded, i.e., there exists $Q > 0$, such
that $ \sup_{\xi} \|G(x,\xi) - g(x)\|_* \leq Q.$

We are interested in the conditional number $\kappa$, which is defined
as $L/\mu$. The conditional number, $\kappa$, could be as large as
$\sqrt{N}$, where $N$ is the number of samples and $T=N$. One
reference case is regularized linear classifiers
\cite{smale03:_estim_approx_error_learn_theor}, where the
regularization factor could be as large as $\sqrt{N}$. The other
reference case is the conditional number of a $N\times n$ random
matrix \cite{rudelson09:_small}, where the smallest singular value is
$O(\sqrt{N}-\sqrt{n-1})$. When $\kappa=\Theta(\sqrt{T})$,
$O(\kappa/T)=O(1/\sqrt{T})$, which bridges the gap between the
convergence rate for strongly convex functions and that for those
without strongly convex condition.  In this note, we assume
$\kappa=O(T)$.  We use big-$O$ notation in term of $T$ and $\kappa$
and hide the factors $D^2L$, $Q^2/L$ and $DQ$ besides constants.

\subsection*{Notation}
Denote $\{1 \cdots T\}$ by $[T]$. Let $\{\xi_t: t\in[T]\}$ be a
sequence of independent random variables.  Denote
$\expect[|t-1]\{\cdot\}:=\expect\{\cdot | \xi_1,\cdots,\xi_{t-1}\}$.
We define $\vln(T,t)=\sum_{\tau=t+1}^T \frac{1}\tau$. Then
$\vln(T,t)\leq\frac{1}{t+1}+\ln(T/(t+1))$, and for $t \geq 1$,
$\vln(T,t)\leq \ln(T/t)$.

\section{Stochastic gradient descent algorithm}
\begin{algorithm}
  \caption{Stochastic gradient descent algorithm}
  \label{alg:1}
  \begin{algorithmic}[1]
    \STATE Input: initial solution $x_0$, step sizes $\{\gamma_t>0: t
    \in [T]\}$ and averaging factor $\{\alpha_t > 0: t \in [T]\}$. %
    \FOR{ $t \in [T]$ } %
    \STATE Let sample gradient $\hat{g}_k= G(x_{t-1},\xi_{t})$, where
    $\xi_{t}$ is independent from $\{\xi_\tau: \tau \in [t-1]\}$. %
    \STATE Let $\displaystyle x_t = \argmin_{x\in \calX}\left\{
      \left<\hat{g}_t, x\right> + \frac{1}{2\gamma_t} \|x-x_{t-1}\|^2
    \right\}$; %
    \STATE Set $\bx_t=\bx_{t-1} + \alpha_t (x_t-\bx_{t-1})$; %
    \ENDFOR %
    \STATE Output: $\bx_T$.
  \end{algorithmic}
\end{algorithm}

Algorithm~\ref{alg:1} shows the stochastic gradient descent
method. Unlike the conventional averaging by equal weights $w_t=1/T$,
we use a weighting scheme $w_t=\alpha_t \prod_{\tau=t+1}^T
(1-\alpha_\tau)=t/(2T(T+1))$, where $\alpha_t=2/(t+1)$ .
Theorem~\ref{thm:1} shows a convergence rate of $O(\kappa/T)$,
assuming that $T>\kappa$.  Let $A_t=\|x_t-x_*\|^2$,
$B_t=\left<\delta_t,x_{t-1}-x_*\right>/Q$, $C_t=\|\delta_t\|_*^2/Q^2$,
and the coefficients $b_t=O(1)$ and $c_t=O(1/t)$.  The informal
argument is that the weighting scheme equalizes the variance of each
iteration, since $\mathrm{var}(b_tB_t)$ and $c_t C_t$ are $O(1/t)$
assuming that $A_t=O(1/t)$.
\begin{thm}
  \label{thm:1}
  Assume that the underlying function $f$ is strongly convex, i.e.,
  $\mu>0$.  Let $\kappa=L/\mu$.  If $\alpha_t=\frac{2}{t+1}$,
  $\gamma_t=\frac{2}{\mu(t+2\kappa)}$, then it holds for
  Algorithm~\ref{alg:1} that for $\theta>0$,
  \begin{equation}
   \label{eq:8} 
   \Pr\{ f(\bx_T) - f(x_*) \geq 
   \bK(T) + \sqrt{2\theta}\tK(T) + \theta \hK(T)\} \leq \exp\{-
   \theta\},
  \end{equation}
    where
    \begin{align*}
      \bK(T)&:= \frac{D^2 L}{T}+ \frac{2 \kappa Q^2}{L T}= O(\kappa/T), \\
      \tK(T)&:= \frac{4DQ (\kappa+1)}{T^{3/2}} +\frac{2\sqrt{2}\kappa
        Q^2 }{L T} +\frac{4\sqrt{2}\kappa^{3/2} Q^2 \sqrt{1+\ln
          T}}{LT^{3/2}}
      =O(\kappa/T), \\
      \hK(T)&:= \frac{10 \kappa Q^2}{L T} =O(\kappa/T).
    \end{align*}
\end{thm}
Similarly with traditional equal weighting scheme, $w_t=1/T$, we have
a convergence rate of $O(\kappa \ln(T)/T)$ in
Proposition~\ref{prop:original}. Informally, $\mathrm{var}(\sum_t w_t
b_tB_t)=\ln(T)/T$ implies a convergence rate of $O(\ln(T)/T)$.
\begin{prop}
  \label{prop:original}
  Assume that $\mu>0$.  Let $\kappa=L/\mu$.  If
  $\alpha_t=\frac{1}{t}$, $\gamma_t=\frac{1}{\mu(t+\kappa)}$, then
  for $\theta>0$,
    $$\Pr\{ f(\bx_T) - f(x_*) \geq 
    \bK(T) + \sqrt{2\theta}\tK(T) + \theta \hK(T)\} \leq \exp\{-
    \theta\},$$ where 
    \begin{align*}
      \bK(T)&:= \frac{LD^2}{2T}+ \frac{ \kappa Q^2}{2L T} (1+\ln T), &
      \tK(T)&:= \frac{D Q \sqrt{\kappa+1}}{T}+\frac{\kappa Q^2}{L T}
      \sqrt{1+\ln T}, &\hK(T)&:=\frac{6 \kappa Q^2}{L T}.
    \end{align*}
\end{prop}

Proposition~\ref{prop:interior} shows that if the optimal solution
$x_*$ is an interior point, it is possible to simply take the
non-averaged solution, $x_T$. The convergence rate is
$O(\kappa^2/T)$. However, if $\kappa=\Theta(\sqrt{T})$,
$O(\kappa^2/T)$ means not convergent, just like the non-averaged SGD
solution without strongly convex conditions.
\begin{prop}
  \label{prop:interior}
  Assume that $\mu>0$ and the optimal solution $x_*$ is an interior
  point.  Let $\kappa=L/\mu$.  If $\gamma_t=\frac{1}{\mu(t+\kappa)}$,
  then
    for $\theta >0$,
    $$\Pr\{ f(x_T) - f(x_*) \geq 
    \bK(T) + \sqrt{2\theta}\tK(T) + \theta \hK(T)\} \leq \exp\{-
    \theta\}, $$ where
    \begin{align*}
      \bK(T)&:= \frac{D^2L (\kappa+1)^2}{2(T+\kappa)^2} +
      \frac{\kappa^2Q^2 (T+\kappa(1+\ln T))}{2L (T+\kappa)^2}
      =O(\kappa^2/T), \\
      \tK(T)&:= \frac{DQ(\kappa+1)^2}{\sqrt{2}(T+\kappa)^{3/2}}
      +\frac{\kappa^2 Q^2}{2L(T+\kappa)} +\frac{\kappa^2 Q^2
        \sqrt{\kappa T(1+\ln(T))}}{2L(T+\kappa)^2}
      =O(\kappa^2/T), \\
      \hK(T)&:=\frac{6\kappa^2 Q^2}{L(T+\kappa)}=O(\kappa^2/T).
    \end{align*}
\end{prop}

\begin{remark}
  There are studies on the high probability convergence rate of
  stochastic algorithm on strongly convex functions, such as
  \cite{rakhlin12:_makin_gradien_descen_optim_stron}.  The convergence
  rate usefully is $O(\mathrm{polylog}(T)/T)$.
  Here, we prove a convergence rate of $O(\frac{\kappa}{T})$ with
  proper weighting scheme.
  % A mistake in Lemma 2 of
  % \cite{rakhlin12:_makin_gradien_descen_optim_stron}.  To induce high
  % probability bound of $X_t$, it requires that all high probability
  % bounds of $X_\tau$ for $\tau\leq t-1$ hold. Thus the probability is
  % reduced from $1-\delta$ to $1-t\delta$, and the overall probability
  % is reduced to $1-T\delta$. By replacing $T\delta$ with $\delta'$,
  % the overall rate becomes $O(\log(\log(T)T/\delta')/T)$, or
  % simplified as $O(\log(T/\delta')/T)$, with probability $1-\delta'$.
\end{remark}

\section{Accelerated Stochastic Gradient Descent Algorithm}
\begin{algorithm}[ht]
  \caption{Accelerated Stochastic Gradient Descent algorithm}
  \label{alg:2}
  \begin{algorithmic}[1]
    \STATE Input: $x_0$, $\mu$, $\{\alpha_t\geq 0\}$,
    $\{\gamma_t>0\}$;%
    \STATE Let $\bx_0=x_0$; %
    \FOR{ $k \in [T]$ } %
    \STATE Let $y_{t-1}=\alpha_t x_{t-1} + (1-\alpha_t) \bx_{t-1}$; %
    \STATE Let $\hat{g}_t=G(y_{t-1},\xi_{t})$, where $\{\xi_{t}\}$ is
    a sample; %
    \STATE\label{step:6} Let $\displaystyle x_t = \argmin_{x\in
      \calX}\left\{ \left<\hat{g}_t-\mu(y_{t-1}-x_{t-1}), x\right> +
      \frac{1}{2\gamma_t} \|x-x_{t-1}\|^2 \right\}$; %
    \STATE Set $\bx_t=\bx_{t-1}+\alpha_t(x_{t}- \bx_{t-1})$; %
    \ENDFOR %
    \STATE Output: $\bx_t$.
  \end{algorithmic}
\end{algorithm}

Algorithm~\ref{alg:2} is a stochastic variant of
Nesterov's accelerated methods.  The convergence rate is also
$O(\kappa/T)$. Comparing with Theorem~\ref{thm:1}, the determinant
part in Theorem~\ref{thm:2} have a better rate,
i.e. $\frac{LD^2}{T^2}$.
\begin{thm}\label{thm:2}
  Assume that $\mu>0$. If $\alpha_t = \frac{2}{t+1}$, $\gamma_t =
  \frac{1}{\mu(2\kappa/t+1/\alpha_t)}$, then for $\theta>0$, 
  $$\Pr\{ f(\bx_T) -
  f(x_*) > \bK(T) + \sqrt{2\theta} \tK(T)+ \theta \hK(T) \} \leq
  \exp\{- \theta\},$$ where
    \begin{align*}
      \bK(T)&:= \frac{2 D^2L}{T^2} + \frac{2 \kappa Q^2}{L T},
      &\tK(T)&:= \frac{\sqrt{20 \kappa} DQ}{T^{3/2}} + \frac{\sqrt{10}
        \kappa Q^2}{2 L T} , &\hK(T)&:=\frac{8\kappa Q^2}{L T}.
    \end{align*}
\end{thm}
\begin{remark}
  The paper \cite{ghadimi12:_optim_stoch_approx_algor_stron} has its
  strongly convex version for AC-SA for sub-Gaussian gradient
  assumption, but its proof relies on a multi-stage algorithm. %
  \par
  Although SAGE
  \cite{hu09:_accel_gradien_method_stoch_optim_onlin_learn} also
  provided a stochastic algorithm based on Nesterov's method for
  strongly convexity, the high probability bound was not given in the
  paper.
\end{remark}

\section{A note on weighting schemes}
In this study, we find the interesting property of weighting scheme
with $\alpha_t=\frac{2}{t+1}$, i.e. $w_t=\frac{2t}{T(T+1)}$. The
scheme takes advantage of a sequence with variance at the decay rate
of $\frac{1}{t}$. Now let informally investigate a sequence with
homogeneous variance, say $1$. With a constant weighting scheme,
$\alpha_t=1/t$, i.e.  $w_t=1/T$, the averaged variance is $1/T$. With
an exponential weighting scheme, $\alpha_1=1$, $\alpha_t=\alpha$,
i.e. $w_1= (1-\alpha)^{T-1}$ and $w_t=\alpha (1-\alpha)^{T-t}$, the
averaged variance is $ \frac{\alpha}{2-\alpha}(1+(1-\alpha)^{2T-1})
\approx \frac{\alpha}{2-\alpha}$, which is translated to that the
number of effective tail samples is a constant $\frac{2}{\alpha}-1$.
With the weighting scheme $\alpha_t=\frac{2}{t+1}$ or
$w_t=2t/(T(T+1))$, the averaged variance is $\frac{2(2T+1)}{3T(T+1)}
\approx \frac{4}{3T}$, which is translated to $\frac{3T}{4}$ effective
tail samples.  This is a trade-off between sample efficiency and
recency. To make other trade-offs, We can use a generalized
scheme\footnote{An alternative scheme is $\alpha_t=\frac{1+r}{t+r}$ or
  $w_t=\frac{(1+r)\Gamma(t+r;t)}{\Gamma(T+r+1;T)}$, where
  $\Gamma(T;t):=\Gamma(T)/\Gamma(t).$},
$\alpha_t=\frac{t^r}{\sum_{\tau=1}^t \tau^r}$ or
$w_t=\frac{t^r}{\sum_{\tau=1}^T \tau^r}$.  Then the averaged variance
is approximately $\frac{(1+r)^2}{(1+2r) T}$.

\section{Proofs}
The proof strategy is first to construct inequalities from the
algorithms in Lemma~\ref{lem:alg1} and \ref{lem:alg2}, then to apply
Lemma~\ref{lem:probineq} to derive the probability inequalities.
\begin{lem}
  \label{lem:probineq}
  Assume that $B_t$ is martingale difference, $w_t \geq 0$, $\pa_t
  \geq 0$, $\pc_t \geq 0$, $a_t\geq 0$, $c_t \geq 0$, $d_t >0$, $A_0
  \leq D^2$, $A_t \geq 0$, and
  \begin{align}
    X_t &= w_t ( \pa_t A_{t-1} +2 \pb_t B_t + \pc_t C_t), \label{eq:14}\\
    A_t &\leq d_t ( a_t A_{t-1} +2 b_t B_t + c_t C_t), \label{eq:15}\\
    B_t^2 & \leq A_{t-1} C_t, \nonumber \\
    C_t & \leq 1. \nonumber %\label{eq:12}
  \end{align}
  If the following conditions hold 
  \begin{enumerate}
  \item for $u\in (0,\frac{1}{2\hR_T})$,
    \begin{equation}
      \label{eq:21}
      \expect[|T] \exp(u X_{T+1}) \leq \exp((u \bP_T +
      \frac{2u^2 \tP_T^2}{1-u\hR_T})A_T +u \bR_T + \frac{2u^2
        \tR_T^2}{1-u\hR_T}),
    \end{equation}
  \item for $t\in[T]$,
    \begin{equation}
      \label{eq:19}
      \begin{aligned}
      \lefteq a_t d_t \bP_t + w_t \pa_t \leq  \bP_{t-1}, \\
      \lefteq \bR_t +w_t \pc_t + c_t d_t \bP_t %
      \leq \bR_{t-1}, %\label{eq:11}
      \\
      \lefteq a_td_t \tP_t^2 + 4 (w_t \pb_t+b_t d_t\bP_t)^2 \leq
      \tP_{t-1}^2 , %\label{eq:13}
      \\
      \lefteq \tR_t^2 + c_t d_t \tP_t^2
      \leq \tR_{t-1}^2, \\
      \lefteq \hR_t \leq \hR_{t-1}, \\
      \lefteq a_t d_t \tP_t^2 \hR_t +4b_t d_t( w_t\pb_t+b_t
      d_t\bP_t)\tP_t^2 \leq
      \tP_{t-1}^2 \hR_{t-1}, \\
      \lefteq a_t d_t \tP_t^2 \hR_t^2 +4b_t d_t( w_t\pb_t+b_t
      d_t\bP_t)\tP_t^2 \hR_t + 2 b_t^2 d_t^2 \tP_t^4 \leq \tP_{t-1}^2
      \hR_{t-1}^2,
    \end{aligned}
    \end{equation}
  \end{enumerate}
  then for $\theta>0$,
  \begin{equation}
    \label{eq:23}
    \Pr\{ \sum_{t=1}^{T+1} X_t \geq 
    \bP_0 D^2+\bR_0 + \sqrt{2\theta}(\tP_0 D^2 + \tR_0) + 2 \theta \hR_0\} \leq \exp\{- \theta\}.
  \end{equation}
\end{lem}
\begin{proof}
  We will prove the following inequality by induction,
  \begin{align}
    \label{eq:22}
    \lefteq \expect[|t]\exp(u \sum_{\tau=t+1}^{T+1} X_\tau) \leq
    \exp((u \bP_t + \frac{2u^2 \tP_t^2}{1-u\hR_t})A_t +u \bR_t +
    \frac{2u^2 \tR_t^2}{1-u\hR_t}), \quad \forall u \in
    (0,\frac{1}{2\hR_t}).
  \end{align}
  Eq.~\ref{eq:21} implies that Eq.~(\ref{eq:22}) holds for $t=T$.  For
  $u\in(0,\frac{1}{2\hR_{t-1}})$,
  \begin{align}
    \lefteq \expect[|t-1]\exp(u \sum_{\tau=t}^{T+1} X_\tau) \leq
    \expect[|t-1]\exp(u X_t + (u \bP_t + \frac{u^2
      \tP_t^2}{2(1-u\hR_t)})A_t +u \bR_t +
    \frac{u^2 \tR_t^2}{2(1-u\hR_t)}) \label{eq:9}\\
    & \leq \expect[|t-1]\exp(u w_t (\pa_t A_{t-1} +2 \pb_t B_t + \pc_t
    C_t) + (u \bP_t + \frac{u^2 \tP_t^2}{2(1-u\hR_t)})d_t
    (a_{t}A_{t-1} + 2 b_t B_t + c_t C_t) +u \bR_t +
    \frac{u^2 \tR_t^2}{2(1-u\hR_t)}) \label{eq:10}\\
    &\leq \exp((u (\bP_t d_t a_t + p_t \pa_t) + \frac{u^2 \tP_t^2 d_t
      a_t}{2(1-u\hR_t)}) A_{t-1} +u(\bR_t +p_t c_t + \bP_t d_t c_t) +
    \frac{u^2 \tR_t^2}{2(1-u\hR_t)} +\frac{ u^2 \tP_t^2d_t
      c_t}{2(1-u\hR_t)} )
    \label{eq:20} \\
    & \quad \times \expect[|t-1]\exp( 2 u (w_t \pb_t + b_td_t \bP_t +
    \frac{u
      b_t d_t\tP_t^2}{2(1-u\hR_t)}) B_t) \nonumber \\
    &\leq \exp((u (\bP_t d_t a_t + p_t \pa_t) + \frac{u^2 \tP_t^2 d_t
      a_t}{2(1-u\hR_t)}
    +2 u^2 (w_t \pb_t + b_t d_t \bP_t + \frac{u b_t d_t\tP_t^2}{2(1-u\hR_t)})^2) A_{t-1} \nonumber \\
    &\quad +u(\bR_t +w_t \pc_t + \bP_t d_t c_t) + \frac{u^2
      \tR_t^2}{2(1-u\hR_t)} +\frac{ u^2 \tP_t^2d_t c_t}{2(1-u\hR_t)}) \label{eq:16}\\
    &\leq \exp((u (\bP_t d_t a_t + p_t \pa_t) + \frac{u^2 \tP_t^2 d_t
      a_t}{2(1-u\hR_t)} +2 u^2 (w_t \pb_t + b_t d_t \bP_t)^2 +
    \frac{u^3 (w_t \pb_t + b_t d_t \bP_t) b_t d_t\tP_t^2}{2(1-u\hR_t)}
    +\frac{2 u^4 b_t^2 d_t^2\tP_t^4}{2(1-u\hR_t)}) A_{t-1} \nonumber\\
    &\quad +u(\bR_t +w_t \pc_t + \bP_t d_t c_t) + \frac{u^2
      \tR_t^2}{2(1-u\hR_t)} +\frac{ u^2 \tP_t^2d_t c_t}{2(1-u\hR_t)}) \label{eq:17}\\
    & \leq \exp((u \bP_{t-1} + \frac{u^2
      \tP_{t-1}^2}{2(1-u\hR_{t-1})})A_{t-1} +u \bR_{t-1} + \frac{u^2
      \tR_{t-1}^2}{2(1-u\hR_{t-1})}), \label{eq:18}
  \end{align}
  where Eq.~(\ref{eq:9}) is due to the assumption of induction;
  Eq.~(\ref{eq:10}) is due to Eq.~(\ref{eq:14},\ref{eq:15});
  Eq.~(\ref{eq:20}) is due to $C_t\leq 1$; Eq.~(\ref{eq:16}) is due to
  $\expect[|t-1]B_t=0$, $B_t^2 \leq A_{t-1} C_t \leq A_{t-1}$, and
  Hoeffding's lemma, thus $\expect[|t-1] \exp(2 v B_t) \leq \exp( 2v^2
  A_{t-1})$; Eq.~(\ref{eq:17}) is due to $\frac{1}{1-u\hR_t} \leq
  \frac{2\hR_{t-1}}{2\hR_{t-1}-\hR_t} \leq 2$; Eq.~(\ref{eq:18}) is
  due to Eqs.~(\ref{eq:19}).  Then for $u \in (0,\frac{1}{2\hR_t})$,
  \begin{align*}
    \lefteq \expect\exp(u \sum_{\tau=1}^{T+1} X_\tau) \leq \exp((u
    \bP_0 + \frac{u^2 \tP_0^2}{2(1-u\hR_0)})A_0 +u \bR_0 + \frac{u^2
      \tR_0^2}{2(1-u\hR_0)}) \leq \exp(u (\bP_0 D^2+\bR_0) +
    \frac{u^2 (\tP_0^2 D^2+ \tR_0^2)}{2(1- 2u\hR_0)}).
  \end{align*}
  Eq.~(\ref{eq:23}) follows Lemma~\ref{lem:0}.
\end{proof}

We prove Lemma~\ref{lem:alg1}, which is the same as Lemma 7 of
\cite{lan08:_effic_method_stoch_compos_optim} except for the strong
convexity.
\begin{lem} 
  \label{lem:alg1}
  Let $\delta_t=G(x_{t-1},\xi_t)-g(x_{t-1})$, $A_t=\|x_t-x_*\|^2$,
  $B_t=\left<\delta_t,x_{t-1}-x_*\right>/Q$,
  $C_t=\|\delta_t\|_*^2/Q^2$.  If $\gamma_t>0$ and $\gamma_t L< 1$, it
  holds for Algorithm~\ref{alg:1} that
  \begin{align*}
    f(x_t)-f(x_*)& \leq \frac{1-\gamma_t\mu}{2\gamma_t} A_{t-1} -
    \frac{1}{2\gamma_t}A_t - Q B_t+ \frac{\gamma_t}{2(1-\gamma_t L)}
    Q^2 C_t.
  \end{align*}
\end{lem}
\begin{proof} Let $d_t=x_t-x_{t-1}$.
\begin{align}
  f(x_t) &\leq f(x_{t-1}) + \left<g(x_{t-1}), d_t\right> + \frac{L}{2}
  \|d_t\|^2 \label{eq:5}\\
  & \leq f(x_*) + \left<g(x_{t-1}), x_{t} - x_*\right> -
  \frac{\mu}{2} \|x_{t-1}-x_*\|^2 + \frac{L}{2}\|d_t\|^2 \label{eq:6}\\
  & = f(x_*)+\left<\hat{g}_t, x_{t}-x_*\right> - \frac{\mu}{2}
  \|x_{t-1} - x_*\|^2 +\frac{L}{2} \| d_t\|^2 -\left<\delta_t,
    x_{t}-x_*\right> \nonumber\\
  &\leq f(x_*)+\frac{1-\gamma_t \mu}{2\gamma_t} \| x_{t-1}-x_*\|^2
  -\frac{1}{2\gamma_t}\|x_{t} - x_*\|^2 -\frac{1-\gamma_tL}{2\gamma_t}
  \| d_t\|^2 -\left<\delta_t,d_t\right>
  -\left<\delta_t, x_{t-1}-x_*\right> \label{eq:7}\\
  &\leq f(x_*)+\frac{1-\gamma_t\mu}{2\gamma_t} \| x_{t-1}-x_*\|^2
  -\frac{1}{2\gamma_t}\|x_{t} - x_*\|^2 +\frac{\gamma_t}{2(1-\gamma_t
    L)} \| \delta_t\|_*^2 -\left<\delta_t, x_{t-1}-x_*\right>.
\end{align}
Eq.~(\ref{eq:5}) is due to the Lipschitz continuity of $f$,
Eq.~(\ref{eq:6}) due to the strong convexity of $f$, Eq.~(\ref{eq:7})
due to the optimality of Step~4.
\end{proof}

\begin{proof}[Proof of Theorem~\ref{thm:1}] Because $\gamma_t
  L=\frac{2\kappa}{t+2\kappa} < 1$, it follows Lemma~\ref{lem:alg1} that
  \begin{align*}
    f(x_t)-f(x_*) & \leq \frac{1-\gamma_t\mu}{2\gamma_t} A_{t-1} -
    \frac{1}{2\gamma_t} A_{t} - Q B_t+ \frac{\gamma_t
      Q^2}{2(1-\gamma_t L)} \\
    &\leq (t+2\kappa-2) \frac{\mu A_{t-1}}{4}-(t+ 2\kappa) \frac{\mu
      A_{t}}{4} - Q B_t + \frac{Q^2}{\mu t}.
  \end{align*}
  As $f(x_t)-f(x_*) \geq \frac{\mu}{2}A_t$ it follows
  Lemma~\ref{lem:alg1} that $$A_t\leq d_t(a_t A_{t-1} + 2 b_t B_t + c_t
  C_t),$$ where $a_{t} =\frac{\mu(t+2\kappa -2)}{4}$, $b_{t} =
  -\frac{Q}{2}$, $c_t = \frac{Q^2}{\mu t}$ and
  $d_t=\frac{4}{\mu(t+2\kappa+2)}$.  Let $w_t=\alpha_t \prod_{\tau=t+1}^T
  (1-\alpha_\tau)= \frac{2t}{T(T+1)}$.  Assume that $\alpha_0=0$
  and $\gamma_0=1$.  Then
  \begin{align*}
    f(\bx_T) - f(x_*) &\leq \sum_{t=1}^T w_t(f(x_t) -
    f(x_*)) %
    \leq \sum_{t=1}^T w_t \left(
      \frac{1-\gamma_t\mu}{2\gamma_t} A_{t-1} - \frac{1}{2\gamma_t}
      A_{t} - Q B_t+ \frac{\gamma_t
        Q^2}{2(1-\gamma_t L)} \right)  \\
    & \leq \sum_{t=1}^T w_t \left( \frac{1-\gamma_t\mu}{2\gamma_t} -
      \frac{w_{t-1}}{2w_t \gamma_{t-1}} \right) A_{t-1} -
    \sum_{t=1}^T w_t Q B_t
    + \sum_{t=1}^T w_t \frac{\gamma_t Q^2}{2(1-\gamma_t L)} \\
    & \leq \sum_{t=1}^T w_t \left( \frac{L}{2t}
      A_{t-1} - Q B_t + \frac{Q^2}{\mu t}
    \right) 
    \leq \frac{L D^2}{T}+\sum_{t=1}^T w_t
    \left(- Q B_t + \frac{Q^2}{\mu t}
    \right). 
  \end{align*}
  Note that we use the factor $A_{t-1}\leq D^2$ for simplicity. %
  Let $\pa_t=0$, $\pb_t=b_t$, $\pc_t=c_t$, $X_{T+1}=\frac{L D^2}{T}$,
  and
  \begin{align*}
    \bP_t &= 0,\\
    \bR_t &=\frac{L D^2}{T}+\frac{2\kappa Q^2(T-t)}{L T^2}, \\
    \tP_t^2 &= \frac{4Q^2(T-t)(t+2\kappa+2)(t+2\kappa-1)}{T^2(T+1)^2},
    \\
    \tR_t^2 &= \frac{Q^4 \kappa^2}{L^2 T^2(T+1)^2} (8 (T-t)(T-t-1) +
    32\kappa T \vln(T,t)),
    \\
    \hR_t &=\frac{5 \kappa Q^2 (T-t)}{L T^2}.
  \end{align*}
  Given the facts that $\kappa \geq 1$, $(t+2\kappa-2)(t+2\kappa-1)
  \leq (t+2\kappa+1)(t+2\kappa-2)$, $(T-t+1)-(T-t) = 1$,
  $(T-t+1)^2-(T-t)^2 \geq 2 (T-t)$, $(T-t+1)^3-(T-t)^3 \geq 3
  (T-t)^2$, the proof of Eq.~(\ref{eq:8}) follows from
  Lemma~\ref{lem:probineq}, because for $t\geq 1$,
  \begin{align*}
    \lefteq a_t d_t \bP_t + w_t \pa_t =0=  \bP_{t-1}, \\
    \lefteq \bR_t +w_t c_t + c_t d_t \bP_t %
    \leq \bR_t+ \frac{2t}{T^2} \frac{Q^2}{\mu t} %
    \leq    \bR_{t-1},  \\
    %%%%%%%%%%%%
    \lefteq a_td_t \tP_t^2 + 4 (w_t\pb_t +b_t d_t\bP_t)^2\leq
    \frac{t+2\kappa-2}{t+2\kappa+2}\tP_t^2 + \frac{4t^2
      Q^2}{T^2(T+1)^2}
    \leq \frac{Q^2}{T^2(T+1)^2}(4(T-t)(t+2\kappa+1)(t+2\kappa-2) + 4t^2)\\
    & \leq \tP_{t-1}^2 -
    \frac{Q^2}{T^2(T+1)^2}(4(t+2\kappa+1)(t+2\kappa-2) - 4t^2)\\
    & = \tP_{t-1}^2 -
    \frac{Q^2}{T^2(T+1)^2}(4(2\kappa-1)t+16\kappa^2-8\kappa-8)
    \leq \tP_{t-1}^2 , \\
    %%%
    \lefteq \tR_t^2 + c_t d_t \tP_t^2 \leq \tR_t^2 +
    \frac{16Q^4(T-t)(t+2\kappa+2)(t+2\kappa-1)}{\mu^2 T^2(T+1)^2
      t (t+2\kappa+2)} \\
    &\leq \frac{Q^4}{\mu^2 T^2(T+1)^2}
    (8(T-t)(T-t-1)+32\kappa T\vln(T,t) +16(T-t) +
    \frac{16(2\kappa-1)}{t}) \leq \tR_{t-1}^2,
  \end{align*}
  and 
  \begin{align*}
    \lefteq a_t d_t \tP_t^2 \hR_t +4b_t d_t (w_t\pb_t+b_t d_t\bP_t
    )\tP_t^2 \leq \frac{4Q^2(T-t)(t+2\kappa-1)(t+2\kappa-2)}{T^4}
    \hR_t +
    \frac{32Q^4t (T-t)(t+2\kappa-1)}{\mu T^6} \\
    &\leq \frac{Q^4}{\mu T^6}(20(T-t)^2(t+2\kappa-1)(t+2\kappa-2)+
    32 t (T-t)(t+2\kappa-1)) \\
    &\leq \tP_{t-1}^2 \hR_{t-1}- \frac{Q^4(T-t)}{\mu T^6}
    (2\times20(t+2\kappa+1)(t+2\kappa-2) - 32 t (t+2\kappa-1)) \\
    &= \tP_{t-1}^2 \hR_{t-1}- \frac{Q^4 (T-t)}{\mu T^6}
    (8t^2-8t +16\kappa(6t-5)+160\kappa^2-80) \leq \tP_{t-1}^2 \hR_{t-1}. \\
    %%%
    \lefteq a_t d_t \tP_t^2 \hR_t^2 +4b_t d_t (w_t\pb_t+b_t
    d_t\bP_t )\tP_t^2 \hR_t+2 b_t^2 d_t^2 \tP_t^4\\
    &\leq \frac{4Q^2(T-t)(t+2\kappa+1)(t+2\kappa-2)}{T^4} \hR_t^2 +
    \frac{32Q^4t (T-t)(t+2\kappa-1)}{\mu T^6} \hR_t + \frac{128
      Q^6(T-t)^2(t+2\kappa-1)^2}{\mu^2 T^{8}} \\
    &\leq \tP_{t-1}^2 \hR_{t-1}^2- \frac{Q^6 (T-t)^2}{\mu^2T^8}(3
    \times 100 (t+2\kappa+1)(t+2\kappa-2)- 160 t
    (t+2\kappa-1)- 128 (t+2\kappa-1)^2) \\
    &= \tP_{t-1}^2 \hR_{t-1}^2- \frac{Q^6(T-t)^2}{\mu^2T^8}(
    12(t-1)^2+368(t-1)(\kappa-1)+688(\kappa-1)^2+508(t-1)+1656(\kappa-1)+368)
    \\
    &\leq \tP_{t-1}^2 \hR_{t-1}^2.
  \end{align*}
\end{proof}

\begin{proof}[Proof of Proposition~\ref{prop:original}]
  Because $\gamma_t L < 1$, it follows Lemma~\ref{lem:alg1} that
  \begin{align*}
    f(x_t)-f(x_*) & \leq \frac{1-\gamma_t\mu}{2\gamma_t} A_{t-1} -
    \frac{1}{2\gamma_t} A_{t} - Q B_t+ \frac{\gamma_t
      Q^2}{2(1-\gamma_t L)} \\
    &\leq (L+\mu (2t-1)) \frac{A_{t-1}}{2}-(L+ 2\mu t) \frac{A_{t}}{2}
    - Q B_t + \frac{Q^2}{4 \mu t}.
  \end{align*}
  As the strong convexity implies that $f(x_t)-f(x_*) \geq
  \frac{\mu}{2}A_t$, it follows Lemma~\ref{lem:alg1} that $$A_t\leq
  d_t(a_t A_{t-1} + 2 b_t B_t + c_t C_t),$$ where $a_{t}
  =\frac{\mu(t+\kappa-1)}{2}$, $b_{t} = -\frac{Q}{2}$, $c_{t} =
  \frac{Q^2}{2\mu t}$ and $d_t=\frac{2}{\mu(t+\kappa+1)}$.  Let
  $w_t=\alpha_t \prod_{\tau=t+1}^T (1-\alpha_\tau)= \frac{1}{T}$.  Assume
  that $\alpha_0=0$ and $\gamma_0=1$.  Then
  \begin{align*}
    f(\bx_T) - f(x_*) &\leq \sum_{t=1}^T w_t(f(x_t) - f(x_*)) %
    \leq \sum_{t=1}^T w_t \left( \frac{1-\gamma_t\mu}{2\gamma_t}
      A_{t-1} - \frac{1}{2\gamma_t} A_{t} - Q B_t+ \frac{\gamma_t
        Q^2}{2(1-\gamma_t L)} \right)  \\
    & \leq \sum_{t=1}^T w_t \left( \frac{1-\gamma_t\mu}{2\gamma_t} -
      \frac{w_{t-1}}{2w_t \gamma_{t-1}} \right) A_{t-1} - \sum_{t=1}^T
    w_t Q B_t
    + \sum_{t=1}^T w_t \frac{\gamma_t Q^2}{2(1-\gamma_t L)} \\
    & \leq \frac{L A_0}{2 T} +\sum_{t=1}^T w_t \left(- Q B_t +
      \frac{Q^2}{4\mu t}
    \right). 
  \end{align*}
  Let $\pa_t=0$, $\pb_t=b_t$, $\pc_t=c_t$, $X_{T+1}=\frac{L D^2}{2T}$,
  and
  \begin{align*}
    \bP_t &=0,\\
    \bR_t &= \frac{Q^2}{2 \mu T} \vln(T,t), \\
    \tP_t^2 &= \frac{Q^2 (t+\kappa+1)}{T^2}, \\
    \tR_t^2 &=\frac{Q^4}{\mu^2 T^2} \vln(T,t), \\
    \hR_t &=\frac{3Q^2}{\mu T}.
  \end{align*}
  The proof follows from Lemma~\ref{lem:probineq}, because for $k \geq 1$,
  \begin{align*}
    \lefteq \bP_t d_t a_t + p_t \pa_t =0 = \bP_{t-1}, \\
    %%%
    \lefteq \bR_t +w_t c_t + \bP_t d_t c_t %
    \leq \frac{ Q^2}{2 \mu T} \ln\frac{T}{t} + \frac{1}{T}
    \frac{Q^2}{2\mu t} %
    \leq    \bR_{t-1},  \\
    %%%
    \lefteq \tP_t^2 d_t a_t + 4 (w_t+\bP_t d_t)^2 b_t^2 \leq \frac{Q^2
      (\kappa+t+1)}{T^2} \frac{t+\kappa-1}{t+\kappa+1} +
    \frac{Q^2}{T^2} =\frac{Q^2 (t+\kappa)}{T^2(t+\kappa+1)}
    = \tP_{t-1}^2 , \\
    %%%
    \lefteq \tR_t^2 + \tP_t^2 d_t c_t %
    \leq\frac{Q^4}{\mu^2 T^2} \ln\frac{T}{t} %
    + \frac{2Q^2}{\mu T^2} \frac{Q^2}{2\mu t} \leq \tR_{t-1}^2,
  \end{align*}
  and
  \begin{align*}
    \lefteq a_t d_t \tP_t^2 \hR_t +4b_t d_t (w_t\pb_t+b_t d_t\bP_t
    )\tP_t^2 \leq \frac{Q^2(t+\kappa-1)(t+\kappa+1)}{T^2(t+\kappa+1)}
    \hR_t + \frac{2Q^4(t+\kappa+1)}{\mu T^3 (t+\kappa+1)}
    \leq \frac{Q^2(t+\kappa)}{T^2} \hR_{t-1}. \\
    \lefteq a_t d_t \tP_t^2 \hR_t^2 +4b_t d_t (w_t\pb_t+b_t
    d_t\bP_t )\tP_t^2 \hR_t+2 b_t^2 d_t^2 \tP_t^4\\
    &\leq \frac{Q^2(t+\kappa-1)(t+\kappa+1)}{T^2(t+\kappa+1)} \hR_t^2
    + \frac{2Q^4(t+\kappa+1)}{\mu T^3 (t+\kappa+1)} \hR_t +
    \frac{2Q^6(t+\kappa+1)^2}{\mu^2 T^4 (t+\kappa+1)^2} \leq
    \frac{Q^2(t+\kappa)}{T^2} \hR_{t-1}^2.
  \end{align*}
\end{proof}

\begin{proof}[Proof of Proposition~\ref{prop:interior}] Because
  $\gamma_t L < 1$, it follows Lemma~\ref{lem:alg1} that
  \begin{align*}
    f(x_t)-f(x_*) & \leq \frac{1-\gamma_t\mu}{2\gamma_t} A_{t-1} -
    \frac{1}{2\gamma_t} A_{t} - Q B_t+ \frac{\gamma_t
      Q^2}{2(1-\gamma_t L)} \\
    &\leq (L+\mu (t-1)) \frac{A_{t-1}}{2}-(L+ \mu t) \frac{A_{t}}{2}
    - Q B_t + \frac{Q^2}{2 \mu t}.
  \end{align*}
  As the strong convexity implies that $f(x_t)-f(x_*) \geq
  \frac{\mu}{2}A_t$, it follows Lemma~\ref{lem:alg1} that $$A_t\leq
  d_t(a_t A_{t-1} + 2 b_t B_t + c_t C_t),$$ where $a_{t}
  =\frac{\mu(t+\kappa-1)}{2}$, $b_{t} = -\frac{Q}{2}$, $c_{t} =
  \frac{Q^2}{2\mu t}$ and $d_t=\frac{2}{\mu(t+\kappa+1)}$.  Because
  the solution is an interior point, we have
  \begin{align*}
    f(x_T)-f(x_*) & \leq \frac{L}{2} A_{T}.
  \end{align*}
  Let $w_t=0$, $X_{T+1}=\frac{L}{2} A_T$, and
  \begin{align*}
    \bP_t&= \frac{L(t+\kappa)(t+\kappa+1)}{2(T+\kappa)(T + \kappa+1)}, \\
    \bR_t&= \frac{\kappa^2 Q^2}{2L(T+\kappa)(T + \kappa+1)}
    (T-t+ \kappa\vln(T,t)), \\
    \tP_t^2&=\frac{Q^2\kappa^2
      (T-t)(t+\kappa)(t+\kappa+1)}{2(T+\kappa)^2(T+\kappa+1)^2} ,
    \\
    \tR_t^2&=\frac{\kappa^4Q^4}{4L^2(T+\kappa)^2(T+\kappa+1)^2}
    ((T-t)(T-t-1)+\kappa T\vln(T,t)),\\
    \hR_t &=\frac{2\kappa^2 Q^2(T-t)}{L(T+\kappa)(T+\kappa+1)}.
  \end{align*}
  The proof follows from Lemma~\ref{lem:probineq}, because
  \begin{align*}
    %%%%%%%%%%%%%%%%%
    \lefteq \bP_t d_t a_t
    = \frac{L(t+\kappa)(t+\kappa-1)}{2(T+\kappa)(T + \kappa+1)} = \bP_{t-1}, \\
    \lefteq \bR_t + \bP_t d_t c_t %
    \leq \bR_t
    +\frac{L(t+\kappa)(t+\kappa+1)}{2(T+\kappa)(T+\kappa+1)}
    \frac{2}{\mu (t+\kappa+1)} \frac{Q^2}{2\mu t} \\
    &\leq \frac{\kappa^2 Q^2}{2 L (T+\kappa)(T+\kappa+1)}(T-t+\kappa
    \vln(T,t) + \frac{t + \kappa}{t})
    \leq    \bR_{t-1},  \\
    %%%%
    \lefteq \tP_t^2 d_t a_t + \bP_t^2 d_t^2 b_t^2 \leq
    \frac{t+\kappa-1}{t+\kappa+1} \tP_t^2+
    \frac{\kappa^2Q^2 (t+\kappa)^2}{4 (T+\kappa)^2(T+\kappa+1)^2}  \\
    &\leq \tP_{t-1}^2- \frac{\kappa^2Q^2}{(T+\kappa)^2(T+\kappa+1)^2}
    (\frac{1}{2}(t+\kappa-1)(t+\kappa)-\frac{1}{4}(t+\kappa)^2) \\
    &\leq \tP_{t-1}^2- \frac{\kappa^2Q^2}{(T+\kappa)^2(T+\kappa+1)^2}
    (\frac{1}{4}(t+\kappa)(t+\kappa-2))
    \leq \tP_{t-1}^2 , \qquad \text{[$t\geq 1$ and $\kappa\geq 1$]} \\
    %%%%
    \lefteq \tR_t^2 + \tP_t^2 d_t c_t \leq \tR_t^2+
    \frac{Q^4\kappa^2(T-t)(t+\kappa)}{2\mu^2(T+\kappa)^2(T+\kappa+1)^2
      t} \\
    & \leq \frac{Q^4\kappa^2}{4\mu^2(T+\kappa)^2(T+\kappa+1)^2}(
    (T-t)(T-t-1) + \kappa T\vln(T,t) +2(T-t)+\frac{(T-t)\kappa}{t})
    \leq \tR_{t-1}^2,
  \end{align*}
  and
  \begin{align*}
    \lefteq a_t d_t \tP_t^2 \hR_t +4b_t^2d_t^2\bP_t \tP_t^2 \\
    &\leq \frac{Q^2 \kappa^2 }{(T+\kappa)^2(T+\kappa+1)^2}
    \left(\frac{1}{2}(T-t)(t+\kappa)(t+\kappa-1)\hR_t +(T-t)(t+\kappa)
      \frac{L Q^2(t+\kappa)}{\mu^2(T+\kappa)(T+\kappa+1)} \right)
    \\
    &\leq \tP_{t-1}^2 \hR_{t-1}- \frac{ Q^4 \kappa^4
    }{L(T+\kappa)^3(T+\kappa+1)^3}
    \left(2 (T-t)(t+\kappa)(t+\kappa-1)-(T-t)(t+\kappa)^2\right) \\
    &\leq \tP_{t-1}^2 \hR_{t-1}- \frac{ Q^4 \kappa^4
    }{L(T+\kappa)^3(T+\kappa+1)^3 (T-t)} (t+\kappa)(t+\kappa-2)
    \leq \tP_{t-1}^2 \hR_{t-1}. \\
    \lefteq a_t d_t \tP_t^2 \hR_t^2 +4b_t^2d_t^2\bP_t \tP_t^2 \hR_t
    +2b_t^2d_t^2 \tP_t^4
    \\
    &\leq \frac{Q^2 \kappa^2 }{(T+\kappa)^2(T+\kappa+1)^2}
    (\frac{1}{2}(T-t)(t+\kappa)(t+\kappa-1)\hR_t^2 +(T-t)(t+\kappa)
    \frac{L Q^2(t+\kappa)}{\mu^2(T+\kappa)(T+\kappa+1)} \hR_t \\
    &\quad
    +\frac{Q^4\kappa^2(T-t)^2(t+\kappa)^2}{4\mu^2(T+\kappa)^2(T+\kappa+1)^2})
    \\
    &\leq \tP_{t-1}^2 \hR_{t-1}^2-\frac{Q^6 \kappa^6
    }{L^2(T+\kappa)^4(T+\kappa+1)^4}
    (6(T-t)^2(t+\kappa)(t+\kappa-1)-2(T-t)^2(t+\kappa)
    -\frac{1}{4}(T-t)^2(t+\kappa)^2)
    \\
    &\leq \tP_{t-1}^2 \hR_{t-1}^2-\frac{Q^6 \kappa^6 (T-t)^2(t+\kappa)
    }{L^2(T+\kappa)^4(T+\kappa+1)^4} (\frac{15}{4}(t+\kappa)-6)
    \leq \tP_{t-1}^2 \hR_{t-1}^2.
  \end{align*}
\end{proof}

Similar to Lemma 9 of \cite{lan08:_effic_method_stoch_compos_optim},
we have the following lemma for Algorithm~\ref{alg:2} with the
consideration of strongly convex cases.
\begin{lem}\label{lem:alg2} 
  Let $\delta_t=G(y_{t-1},\xi_t)-g(y_{t-1})$, $A_t=\|x_t-x_*\|^2$,
  $B_t=\left<\delta_t,x_{t-1}-x_*\right>/Q$,
  $C_t=\|\delta_t\|_*^2/Q^2$.  If $0<\alpha_t<1$, $\gamma_t>0$ and
  $\gamma_t(\alpha_t L+\mu)< 1$, it holds for Algorithm~\ref{alg:2}
  that
  \begin{equation*}
    f(\bx_{t})-f(x_*)
    \leq (1-\alpha_t)(f(\bx_{t-1})-f(x_*))
    +\frac{\alpha_t(1-\gamma_t\mu)}{2\gamma_t}A_{t-1} -
    \frac{\alpha_t}{2\gamma_t} A_t 
    -\alpha_tQB_t + \frac{\alpha_t\gamma_t}{2(1-\alpha_t\gamma_t
      L-\gamma_t\mu)}Q^2 C_t.
  \end{equation*}
\end{lem}
\begin{proof}
  Let $d_t=x_{t}-x_{t-1}$ and $v_t=x_{t-1}+\gamma_t\mu (y_{t-1}-x_{t-1})$.
  Note that $ \bx_t-y_{t-1}=\alpha_t d_t.$
  \begin{align}
    f(\bx_{t}) &\leq
    f(y_{t-1})+\left<g(y_{t-1}),\bx_t-y_{t-1}\right>+\frac{L}{2}\|\bx_t-y_{t-1}\|^2
    \label{eq:1} \\
    &=
    (1-\alpha_t)[f(y_{t-1})+\left<g(y_{t-1}),\bx_{t-1}-y_{t-1}\right>]
    +\alpha_t[f(y_{t-1})+\left<g(y_{t-1}),x_{t}-y_{t-1}\right>]
    +\frac{\alpha_t^2L}{2}\|d_t\|^2
    \nonumber \\
    &\leq (1-\alpha_t)f(\bx_{t-1}) + \alpha_tf(x_*)
    +\alpha_t\left<g(y_{t-1}),x_{t}-x_*\right>
    -\frac{\alpha_t\mu}{2} \|y_{t-1} - x_*\|^2
    +\frac{\alpha_t^2L}{2}\|d_t\|^2 
    \label{eq:2} \\
    &= (1-\alpha_t)f(\bx_{t-1}) + \alpha_tf(x_*)
    +\alpha_t\left<\hat{g}_t,x_{t}-x_*\right>
    -\frac{\alpha_t\mu}{2} \|y_{t-1} - x_*\|^2
    +\frac{\alpha_t^2L}{2}\|d_t\|^2 
    -\alpha_t\left<\delta_t,x_{t}-x_*\right>
    \nonumber \\
    &\leq (1-\alpha_t)f(\bx_{t-1})
    +\alpha_t f(x_*)
    +\frac{\alpha_t}{\gamma_t} \left<x_t-v_t, x_*-x_t\right>
    -\frac{\alpha_t\mu}{2} \|y_{t-1} - x_*\|^2 
    +\frac{\alpha_t^2L}{2}\|d_t\|^2 
    -\alpha_t\left<\delta_t,x_{t}-x_*\right>
    \label{eq:3} \\
    &= (1-\alpha_t)f(\bx_{t-1})
    +\alpha_t f(x_*)
    +\frac{\alpha_t(1-\gamma_t\mu)}{2\gamma_t}\|x_{t-1} -x_*\|^2
    -\frac{\alpha_t}{2\gamma_t}\|x_t-x_* \|^2 
    -\frac{\alpha_t\mu}{2}\|y_{t-1}-x_t\|^2
    \nonumber \\
    &\quad
    -\frac{\alpha_t(1- \gamma_t\mu-\alpha_t \gamma_t L)}{2\gamma_t}\|d_t\|^2
    -\alpha_t\left<\delta_t,d_t\right> 
    -\alpha_t\left<\delta_t,x_{t-1}-x_*\right> \nonumber \\
    &\leq (1-\alpha_t)f(\bx_{t-1})
    +\alpha_t f(x_*)
    +\frac{\alpha_t(1-\gamma_t\mu)}{2\gamma_t}\|x_{t-1} -x_*\|^2
    -\frac{\alpha_t}{2\gamma_t}\|x_t-x_* \|^2 
    \nonumber   \\
    &\quad
    +\frac{\alpha_t\gamma_t}{2(1-\gamma_t\mu - \alpha_t\gamma_t L )}\|\delta_t\|_*^2
    -\alpha_t\left<\delta_t,x_{t-1}-x_*\right>.
  \end{align}
  Eq.~(\ref{eq:1}) is due to the Lipschitz continuity of $f$,
  Eq.~(\ref{eq:2}) due to the strong convexity of $f$,
  Eq.~(\ref{eq:3}) due to the optimality of Step~\ref{step:6}.
\end{proof}

\begin{proof}[Proof of Theorem~\ref{thm:2}]
  Let $\lambda_t = \prod_{\tau=t+1}^T
  (1-\alpha_t)=\frac{t(t+1)}{T(T+1)}$. We have and
  \begin{align*}
    \frac{\lambda_{t} \alpha_t(1-\gamma_t\mu)}{\gamma_t}
    -\frac{\lambda_{t-1}\alpha_{t-1}}{\gamma_{t-1}}
    &=\frac{2t}{T(T+1)}(\frac{2L}{t}+\frac{\mu(t+1)}{2}-\mu)
    -\frac{2(t-1)}{T(T+1)}(\frac{2L}{t-1}+\frac{\mu t}{2}) =0, \quad
    \forall t>1.
  \end{align*}
  Let $a_t=\frac{\mu (4\kappa+ t (t-1))}{2t}$, $b_t= -\frac{Q}{2}$,
  $c_t=\frac{Q^2}{\mu t}$, and $d_t= \frac{2t}{\mu (4\kappa+ t
    (t+1))}$.  Summing up the inequality in Lemma~\ref{lem:alg2} weighted
  by $\lambda_t$, we have
  \begin{equation}
    \label{eq:4}
  \begin{aligned} 
    f(\bx_t)-f(x_*) &\leq \frac{\lambda_{1}
      \alpha_1(1-\gamma_1\mu)}{2\gamma_1}A_0 - \frac{\lambda_t
      \alpha_t}{2\gamma_t}A_t + \sum_{\tau=1}^t \lambda_\tau
    \alpha_\tau(-Q B_t+ \frac{\gamma_\tau}{2 (1- \alpha_\tau
      \gamma_\tau L-\gamma_\tau \mu)} Q^2 C_\tau)
    \\
    &\leq \frac{2 L}{T(T+1)} A_0 - \frac{2t}{T(T+1)}\frac{A_t}{d_t}
    +\sum_{\tau=1}^t \frac{2 \tau}{T(T+1)}\left( 2 b_\tau B_\tau+
      c_\tau C_\tau\right).
  \end{aligned}
  \end{equation}
  Let $\tA_t := \frac{d_t}{t}\left\{ L A_0 +
    \sum_{\tau=1}^t \left( 2 \tau b_\tau B_t+ \tau c_t C_t\right)\right\}$.
  Because $f(\bx_t)-f(x_*) \geq 0$, we have
  \begin{align*}
    \frac{t}{d_t} A_t \leq \frac{t}{d_t} \tA_t = \frac{t-1}{d_{t-1}}
    \tA_{t-1} + 2 t b_t B_t + t c_tC_t = t a_t \tA_{t-1} + 2 t b_t B_t
    + t c_tC_t
  \end{align*}
  Then 
  \begin{equation}
    \label{eq:8}
    A_t \leq \tA_t = d_t(a_t\tA_{t-1} + 2b_t B_t + c_t C_t).
  \end{equation}
  Given Eq.~(\ref{eq:4}) and Eq.~(\ref{eq:8}),
  letting $w_t=\frac{2t}{T(T+1)}$, $\pa_t=0$, $\pb_t=b_t$, $\pc_t=c_t$,
  $X_{T+1}=\frac{2L D^2}{T(T+1)}$, and
  \begin{align*}
    \bP_t &= 0,\\
    \bR_t &=\frac{2L D^2}{T^2}+\frac{2\kappa Q^2(T-t)}{L T^2}, \\
    \tP_t^2 &= \frac{5Q^2 (T-t) (t(t+1)+4\kappa)}{T^4}, \\
    \tR_t^2 &= \frac{5\kappa^2 Q^4(T-t)(T-t-1)}{2L^2  T^4 },\\
    \hR_t &=\frac{4 \kappa Q^2 (T-t)}{LT^2},
  \end{align*}
  the proof follows from Lemma~\ref{lem:probineq}, because
  \begin{align*}
    \lefteq a_t d_t \bP_t + w_t \pa_t =0=  \bP_{t-1}, \\
    \lefteq \bR_t +w_t \pc_t + c_t d_t \bP_t %
    \leq \bR_t+ \frac{2t}{T^2} \frac{Q^2}{\mu t} %
    \leq    \bR_{t-1},  \\
    \lefteq a_td_t \tP_t^2 + 4 (w_t\pb_t +b_t d_t\bP_t)^2\leq
    \frac{t(t-1)+4\kappa}{t(t+1)+4\kappa}\tP_t^2 + \frac{4t^2
      Q^2}{T^4} \\
    & \leq \frac{Q^2}{T^4}(6(t(t-1)+4\kappa)(T-t)+4t^2)
    \leq\tP_{t-1}^2 - \frac{Q^2}{T^4}(5 (t(t-1)+4\kappa)-4t^2) \\
    & \leq\tP_{t-1}^2 - \frac{Q^2}{T^4}(t^2 -5t+20\kappa)
    \leq\tP_{t-1}^2 - \frac{Q^2}{T^4}(3 t) \leq \tP_{t-1}^2 , \\
    \lefteq \tR_t^2 + c_t d_t \tP_t^2 \leq \tR_t^2 +
    \frac{5Q^4(T-t)}{\mu^2 T^4 } \leq \tR_{t-1}^2,
  \end{align*}
  and 
  \begin{align*}
    \lefteq a_t d_t \tP_t^2 \hR_t +4b_t d_t (w_t \pb_t + b_t d_t
    \bP_t)\tP_t^2 \leq \frac{Q^2 (t(t-1)+4\kappa)(T-t)}{T^4}
    \hR_t + \frac{4 t^2 Q^4 (T-t)}{\mu T^6} \\
    & \leq
    \frac{Q^4 }{\mu T^6}(4 (t(t-1)+4\kappa)(T-t)^2+ 4 t^2 (T-t)) \\
    & \leq \tP_{t-1}^2 \hR_{t-1}-
    \frac{Q^4 (T-t)}{\mu T^6}(2 \times 4 (t(t-1)+4\kappa)- 4 t^2 ) \\
    & = \tP_{t-1}^2 \hR_{t-1}- \frac{Q^4 (T-t)}{\mu T^6}(4 t^2 -8 t +
    32 \kappa) \leq \tP_{t-1}^2 \hR_{t-1}- \frac{Q^4 (T-t)}{\mu
      T^6}(14 t)
    \leq \tP_{t-1}^2 \hR_{t-1} \\
    %%%
    \lefteq a_t d_t \tP_t^2 \hR_t^2 +4b_t d_t (w_t \pb_t + b_t d_t
    \bP_t)\tP_t^2 \hR_t+2 b_t^2 d_t^2 \tP_t^4\\
    &\leq \frac{5 Q^2 (t(t-1)+4\kappa)(T-t)}{T^4} \hR_t^2 + \frac{20
      t^2 Q^4 (T-t)}{\mu T^6} \hR_t
    + \frac{100 t^2 Q^6 (T-t)^2}{\mu^2 T^8}  \\
    &\leq \frac{Q^6}{\mu^2T^8} (80(t(t-1)+4\kappa)(T-t)^3 +
    80 t^2(T-t)^2+100 t^2(T-t)^2) \\
    &\leq \tP_{t-1}^2 \hR_{t-1}^2 -\frac{Q^6 (T-t)^2}{\mu^2T^8}
    (3 \times 80 (t(t-1)+4\kappa) - 80 t^2-100 t^2) \\
    &= \tP_{t-1}^2 \hR_{t-1}^2 -\frac{Q^6(T-t)^2}{\mu^2T^8} (60t^2-240
    t +960\kappa) \leq \tP_{t-1}^2 \hR_{t-1}^2
    -\frac{(T-t)^2Q^6}{\mu^2T^8} (240 t) \leq \tP_{t-1}^2 \hR_{t-1}^2.
  \end{align*}
\end{proof}
\appendix
\section*{Supporting lemma}
We use part of the proof of Lemma 8 in \cite{birg98:_minim}.
\begin{lem}\label{lem:0}
  Let $B>0$ and $\sigma>0$.  If the log-moment generating function
  satisfies
  \begin{align*}
    \log \expect \exp\{u Z\} &\leq \frac{\sigma^2 u^2}{2 (1-uB)}
    \quad \text{for all $0 \leq u < 1/B$}, 
  \end{align*}
  then
  \begin{equation}
    \label{eq:1}
    \Pr\{ Z \geq \epsilon\} \leq
    \exp\{-\frac{\epsilon^2}{2 \sigma^2 + 2 \epsilon B}\}
    \quad\text{for all $\epsilon \geq 0$},
  \end{equation}
  and
  \begin{equation}
    \label{eq:2}
    \Pr\{ Z \geq \sqrt{2\theta \sigma^2} + \theta B\} \leq
    \exp\{-\theta\}
    \quad \text{for all $\theta \geq 0$}.
  \end{equation}
\end{lem}
\begin{proof} It follows Markov's inequality that
  \begin{align*}
    \Pr\{ Z \geq \epsilon\} & \leq \inf_u \expect \exp\{ -u \epsilon +
    u Z\} = \exp\{-h(\epsilon)\},
  \end{align*}
  where $h(\epsilon):=\sup_{u} u\epsilon - \frac{\sigma^2 u^2}{2
    (1-u B)}$.  Also, the supremum is achieved for
  \begin{equation*}
    \epsilon= \frac{\sigma^2u}{1-u B}+
    \frac{\sigma^2u^2 B}{2(1-u B)^2}=\frac{\sigma^2u}{2(1-uB)}+
    \frac{\sigma^2u}{2(1-u B)^2},
  \end{equation*}
  i.e. $u=B^{-1}[1-\sigma(2\epsilon B + \sigma^2)^{-1/2}] <
  1/B$. Then we prove Eq.~(\ref{eq:1}), as
  \begin{equation*}
    h(\epsilon)=\frac{\epsilon^2}{\epsilon B
      +\sigma^2+\sigma^2(1+2\epsilon
      B/\sigma^2)^{1/2}} \geq \frac{\epsilon^2}{2 \epsilon B + 2\sigma^2}.
  \end{equation*}
  Let
  \begin{equation*}
    \theta:=\frac{\sigma^2u^2}{2(1-u B)^2}
    =h(\epsilon).
  \end{equation*}
   Then we prove Eq.~(\ref{eq:2}), as
   \begin{align*}
    \sqrt{2\theta \sigma^2}+\theta B = \frac{\sigma^2u}{(1-u B)}+
    \frac{\sigma^2u^2 B}{2(1-u B)^2} = \epsilon.
  \end{align*}
\end{proof}

\bibliographystyle{mlapa}
\bibliography{zsh}

\end{document}